\documentclass[11pt]{article} 

\usepackage{url}
\usepackage{graphicx}
\usepackage{amsmath}
\usepackage{amssymb}
\usepackage{amsthm}
\usepackage{algorithm}
\usepackage{algorithmic}
\usepackage{natbib}

\usepackage{fullpage}

\newtheorem{theorem}{Theorem}
\newtheorem{lemma}[theorem]{Lemma}

\newtheorem{corollary}[theorem]{Corollary}
\newtheorem{definition}{Definition}
\newtheorem{example}{Example}

\newcommand{\R}{\mathbb{R}}

\newcommand{\E}{\mathbb{E}}
\DeclareMathOperator{\scf}{sc}
\DeclareMathOperator{\spn}{span}
\DeclareMathOperator{\AV}{AV}
\DeclareMathOperator{\MC}{MC}
\DeclareMathOperator{\CC}{CC}
\DeclareMathOperator{\PAV}{PAV}
\DeclareMathOperator{\SAV}{SAV}
\DeclareMathOperator{\prm}{Pr}

\allowdisplaybreaks
\title{\bf Evaluating Approval-Based Multiwinner Voting\\in Terms of Robustness to Noise\thanks{A preliminary version of the results of this paper with the same title appeared in Proceedings of the 29th International Joint Conference on Artificial Intelligence (IJCAI), pages 74--80, 2020.}}

\author{Ioannis~Caragiannis\thanks{Department of Computer Science, Aarhus University, {\AA}bogade 34, 8200 Aarhus N, Denmark. Email: \texttt{iannis@cs.au.dk}.}
\and Christos~Kaklamanis\thanks{Computer Technology Institute ``Diophantus'' \& Department of Computer Engineering and Informatics, University of Patras, 26504 Rion-Patras, Greece. Email: \texttt{kakl@ceid.upatras.gr}.}
\and Nikos~Karanikolas\thanks{Department of Computer Engineering and Informatics, University of Patras, 26504 Rion-Patras, Greece. Email: \texttt{nkaranik@ceid.upatras.gr}.}
\and George~A.~Krimpas\thanks{Department of Computer Engineering and Informatics, University of Patras, 26504 Rion-Patras, Greece. Email: \texttt{krimpas@ceid.upatras.gr}.}}

\date{}

\begin{document}

\maketitle

\begin{abstract}
Approval-based multiwinner voting rules have recently received much attention in the Computational Social Choice literature. Such rules aggregate approval ballots and determine a winning committee of alternatives. To assess effectiveness, we propose to employ new noise models that are specifically tailored for approval votes and committees. These models take as input a ground truth committee and return random approval votes to be thought of as noisy estimates of the ground truth. A minimum robustness requirement for an approval-based multiwinner voting rule is to return the ground truth when applied to profiles with sufficiently many noisy votes. Our results indicate that approval-based multiwinner voting can indeed be robust to reasonable noise. We further refine this finding by presenting a hierarchy of rules in terms of how robust to noise they are.
\medskip

\end{abstract}

\section{Introduction}\label{sec:intro}
Voting has received much attention by the AI and Multiagent Systems community recently, mostly due to its suitability for simple and effective decision making. One popular line of research, that originates from~\citet{A51}, has aimed to characterize voting rules in terms of the {\em social choice axioms} they satisfy. Another approach views voting rules as {\em estimators}. It assumes that there is an objectively correct choice, a {\em ground truth}, and votes are noisy estimates of it. Then, the main criterion for evaluating a voting rule is whether it can determine the ground truth as outcome when applied to noisy votes.

A typical scenario in studies that follow the second approach employs a hypothetical {\em noise model} that uses the ground truth as input and produces random votes. Then, a voting rule is applied on profiles of such random votes and is considered effective if it acts as a {\em maximum likelihood estimator}~\citep{CS05,Y88} or if it has {\em low sample complexity}~\citep{CPS16}. As such evaluations are heavily dependent on the specifics of the noise model, relaxed effectiveness requirements, such as the {\em accuracy in the limit}, sought in broad classes of noise models~\citep{CPS16} can be more informative. 

We restrict our attention to {\em approval voting}, where ballots are simply sets of alternatives that are approved by the voters~\citep{LS10}. Furthermore, we consider {\em multiwinner voting rules}~\citep{FSST17}, which determine committees of alternatives as outcomes~\citep{K10,ABC+17}. In particular, we focus on approval-based counting choice rules (or, simply, {\em ABCC rules}), which were defined recently by~\citet{LS18}. A famous rule in this category is known as multiwinner approval voting (AV). Each alternative gets a point every time it appears in an approval vote and the outcome consists of a fixed number of alternatives with the highest scores.

We consider noise models that are particularly tailored for approval votes and committees. These models use a committee as ground truth and produce random sets of alternatives as votes. We construct broad classes of noise models that share a particular structure, parameterized by {\em distance metrics} defined over sets of alternatives. In this way, we adapt to approval-based multiwinner voting the approach of~\citet{CPS16} for voting rules over rankings.

\begin{figure*}[t]
	\centering
\includegraphics[trim={40 11.1cm 0 0}, clip, scale=0.47]{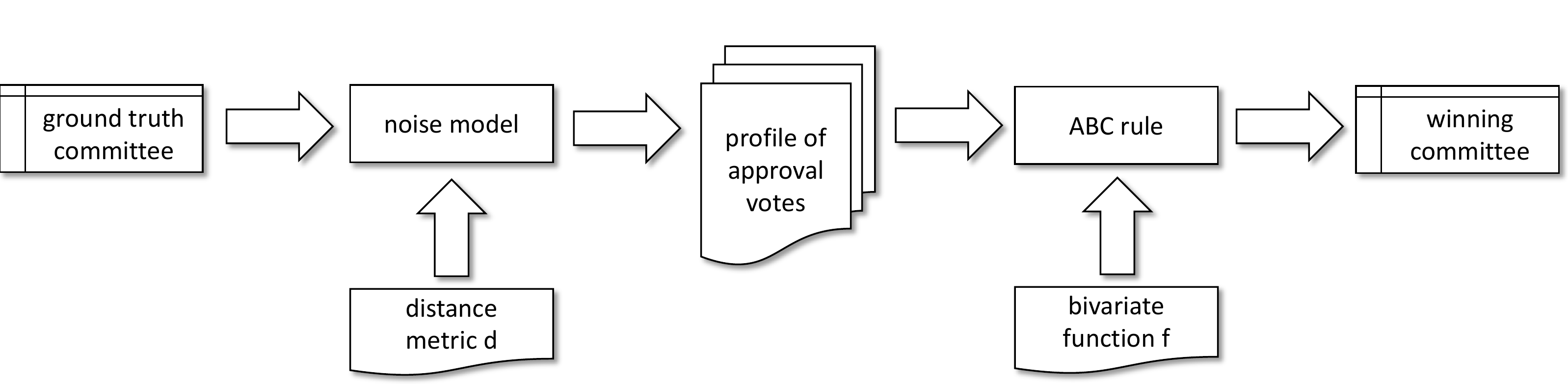}
	\caption{Our evaluation framework.}
	\label{fig}
\end{figure*}

Figure~\ref{fig} illustrates our evaluation framework. The noise model is depicted at the left. It takes as input the ground truth committee and its probability distribution over approval votes {which} is consistent to a distance metric $d$. Repeated executions of the noise model produce a profile of random approval votes. The ABCC rule (defined using a bivariate function $f$; see Section~\ref{sec:prelim}) is then applied on this profile and returns {one or more winning committees}. Our requirement for the ABCC rule is to be accurate in the limit {(informally, on profiles with infinitely many votes, it must return the ground truth as the unique winning committee),} not only for a single noise model, but for all models that belong to a sufficiently broad class. The breadth of this class quantifies the robustness of the ABCC rule to noise.

The details of our framework are presented in Section~\ref{sec:prelim}. Our results indicate that it indeed allows for a classification of ABCC rules in terms of their robustness to noise. In particular, we identify (in Section~\ref{sec:mc}) the modal committee rule (MC) as the ultimately robust ABCC rule: MC is robust against all kinds of reasonable noise. AV follows in terms of robustness and seems to outperform other known ABCC rules (see Section~\ref{sec:av}). In contrast, the well-known approval Chamberlin-Courant (CC) rule is the least robust. On the other hand, all ABCC rules are robust if we restrict noise sufficiently (see Section~\ref{sec:abc}).
We conclude with a discussion on extensions and open problems in Section~\ref{sec:open}.

\subsection{Further related work}
Approval-based multiwinner voting rules have been studied in terms of their computational complexity~\citep{AGG+15,SFL16}, axiomatic properties~\citep{SEL+17,LS18,ABC+17}, as well as their applications~\citep{BLS17}. In particular, axiomatic work has focused on two different principles that govern multiwinner rules: diversity and individual excellence. \citet{LS19} attempt a quantification of how close an approval-based multiwinner voting rule is to these two principles. 
We remark that the primary focus of the current paper is on individual excellence, since a ground truth committee can be interpreted as the ``excellent'' choice in a very natural way.

The robustness of approval voting has been previously evaluated against noise models, using either the MLE~\citep{PS15} or the sample complexity~\citep{CM17} approach. These papers assume a ranking of the alternatives as ground truth, generate approval votes that consist of the top alternatives in rankings produced according to the noise model of \citet{M57}, and assess how well approval voting recovers the ground truth ranking. We believe that our framework is fairer to approval votes, as recovering an underlying ranking when voters have very limited power to rank is very demanding. The robustness of (non-approval) multiwinner voting against noise has been studied by \citet{PRS12}. Different notions of robustness in multiwinner voting are considered by \citet{GF19} and~\citet{BFKNST17}.

Additional references related to specific ABCC rules are given in the next section. We remark that the modal committee (MC) rule is similar in spirit to the modal ranking rule considered by~\citet{CPS14}.

\section{Preliminaries}\label{sec:prelim}
Throughout the paper, we denote by $A$ the set of alternatives. We use $m=|A|$ and denote the committee size by $k$. The term committee refers to a set of exactly $k$ alternatives.

\subsection{Approval-based multiwinner voting}
An approval vote is simply a subset of the alternatives (of any size). An approval-based multiwinner voting rule takes as input a profile of approval votes and returns one or more winning committees.

We particularly consider voting rules that belong to the class of approval-based counting choice rules (or, simply, ABCC rules), introduced by~\citet{LS18}. Such a rule is defined by a bivariate function $f$, with $f(x,y)$ indicating a non-negative score a committee gets from an approval vote containing $y$ alternatives, $x$ of which are common with the committee. $f$ is non-decreasing in its first argument. Formally, $f$ is defined on the set $\mathcal{X}_{m,k}$, which consists of all pairs $(x,y)$ of possible values of $|U\cap S|$ and $|S|$, given that $U$ is $k$-sized and $S$ can be any subset of the $m$ alternatives of $A$. I.e., $\mathcal{X}_{m,k}$ is the set
\begin{align*}
\left\{(x,y): y=0, 1, ..., m, x=\max\{k+y-m,0\}, ..., \min\{y,k\} \right\}.
\end{align*}

The score of a committee is simply the total score it gets from all approval votes in a profile. Winning committees are those that have maximum score. We extensively use ``the ABCC rule $f$'' to refer to the ABCC rule that uses the bivariate function $f$. We denote the score that an ABCC rule $f$ assigns to the committee $U$ given a profile $\Pi=(S_i)_{i\in [n]}$ of $n$ votes by $\scf_f(U,\Pi)=\sum_{i=1}^n{f(|U\cap S_i|,|S|)}$. With some abuse of notation, we use $\scf_f(U,S_i)$ to refer to the score $U$ gets from vote $S_i$. Hence, $\scf_f(U,\Pi)=\sum_{i=1}^n{\scf_f(U,S_i)}$.

Well-known ABCC rules include:
\begin{itemize}
	\item Multiwinner approval voting (AV), which uses the function $f_{\AV}(x,y)=x$.
	\item Approval Chamberlin-Courant (CC), which uses the function $f_{\CC}(x,y)=\min\{1,x\}$. The rule falls within a more general context considered by~\citet{CC83}.
	\item Proportional approval voting (PAV), which uses the function $f_{\PAV}(x,y)=\sum_{i=1}^x{1/i}$.
\end{itemize}
These rules belong to the class of rules that originate from the work of~\citet{T95}. A Thiele rule uses a vector $\langle w_1, w_2, ..., w_k\rangle$ of non-negative weights to define $f(x,y)=\sum_{i=1}^x{w_j}$. Other known Thiele rules include the $p$-Geometric rule~\citep{SFL16} and Sainte Lagu\"e  approval voting~\citep{LS19}.

A well-known non-Thiele rule is the satisfaction approval voting (SAV) rule that uses $f_{\SAV}(x,y)=x/y$ for $y>0$ and $f(x,y)=0$ otherwise~\citep{BK14}. Let us also introduce the {\em modal committee} (MC) rule which returns the committee (or committees) that has maximum number of appearances as approval votes in the profile. MC is also non-Thiele; it uses $f(k,k)=1$ and $f(x,y)=0$ otherwise.

\subsection{Noise models}
We employ noise models to generate approval votes, assuming that the ground truth is a committee. Denoting the ground truth by $U\subseteq A$, a noise model $\mathcal{M}$ produces random approval votes according to a particular distribution that defines the probability $\prm_\mathcal{M}[S|U]$ to generate the set $S\subseteq A$ when the ground truth is $U$.

Let us give the following noise model $\mathcal{M}_p$\footnote{Even though it might look as a toy example of a noise model, a more careful look will reveal that $\mathcal{M}_p$ can be seen as the analog of the famous Mallows noise model \citep{M57} in the classical social choice setting when each voter provides a strict ranking of the alternatives instead of an approval set. Interestingly, the ABCC rule AV turns out to be a maximum likelihood estimator for $\mathcal{M}_p$ (analogously to the fact that the well-known Kemeny rule is an MLE for Mallows; e.g., see~\citealp{Y88}). As this is beyond the scope of the current paper, we present a proof in Appendix.} as an example. $\mathcal{M}_p$ uses a parameter $p\in (1/2,1]$. Given a ground truth committee $U$, $\mathcal{M}_p$ generates a random set $S\subseteq A$ by selecting each alternative of $U$ with probability $p$ and each alternative in $A\setminus U$ with probability $1-p$. Intuitively, the probability that a set will be generated depends on its ``distance'' from the ground truth: the higher this distance, the smaller this probability. To make this formal, we will need the {\em set difference}\footnote{Viewing sets as strings, the distance metric $d_\Delta$ is equivalent to the Hamming distance; see~\citet{DD16}.} distance metric $d_{\Delta}:2^A\rightarrow \R_{\geq 0}$ defined as $d_\Delta(X,Y)=|X\setminus Y|+|Y\setminus X|$. 

\begin{lemma}
	For $S\subseteq A$, $\prm_{\mathcal{M}_p}[S|U]=p^m \cdot \left(\frac{1-p}{p}\right)^{d_\Delta(U,S)}$.
\end{lemma}

\begin{proof}
By the definition of the noise model $\mathcal{M}_p$, the set $S$ is generated by the noise model $\mathcal{M}_p$ with ground truth committee $U$ when each alternative in $S\cap U$ is selected (this happens with probability $p$ independently for each alternative of the set), each alternative in $S\setminus U$ is selected (probability $1-p$ each), each alternative in $U\setminus S$ is not selected (probability $1-p$ each), and each alternative in $A\setminus (S\cup U)$ is not selected (probability $p$ each). Overall, 
\begin{align*}
    \prm_{\mathcal{M}_p}[S|U] &= p^{|S\cap U|}\cdot (1-p)^{|S\setminus U|}\cdot (1-p)^{|U\setminus S|}\cdot p^{|A\setminus (S\cup U)|}\\
    &=p^{m-d_{\Delta}(U,S)}\cdot (1-p)^{d_\Delta(U,S)}=p^m \cdot \left(\frac{1-p}{p}\right)^{d_\Delta(U,S)}
\end{align*}
as desired.
\end{proof}

So, since $p>1/2$, the probability $\prm_{\mathcal{M}_p}[S|U]$ is decreasing in $d_\Delta(U,S)$. We will consider general noise models $\mathcal{M}$ with $\prm_\mathcal{M}[S|U]$ depending on $d(U,S)$, where $d$ is a distance metric defined over subsets of $A$.

\begin{definition}\label{defn:d-mon}
	Let $d$ be a distance metric over sets of alternatives. A noise model $\mathcal{M}$ is called $d$-monotonic if for any two sets $S_1, S_2\subseteq A$, it holds $\prm_{\mathcal{M}}[S_1|U]>\prm_{\mathcal{M}}[S_2|U]$ if and only if $d(U,S_1)<d(U,S_2)$.
\end{definition}
Definition~\ref{defn:d-mon} implies that $\prm_{\mathcal{M}}[S_1|U]=\prm_{\mathcal{M}}[S_2|U]$ when $d(U,S_1)=d(U,S_2)$.

Besides the set difference metric used by $\mathcal{M}_p$, other well-known distance metrics \citep[see][]{DD16} are:
\begin{itemize}
	\item the normalized set difference or {\em Jaccard metric} $d_J$, defined as $d_J(X,Y)=\frac{d_\Delta(X,Y)}{|X\cup Y|}$,
	\item the maximum difference or {\em Zelinka metric} $d_Z$, defined as $d_Z(X,Y)=\max\{|X\setminus Y|,|Y\setminus X|\}$, and
	\item the normalized maximum difference or {\em Bunke-Shearer metric}	$d_{BS}$, defined as $d_{BS}(X,Y)=\frac{d_Z(X,Y)}{\max\{|X|,|Y|\}}$.
\end{itemize}
Notice that $d(X,Y)$ for the four specific distance metrics defined here depends only on $|X\setminus Y|$, $|Y\setminus X|$, $|X|$, and $|Y|$. Thus, these distance metrics are {\em alternative-independent}. Our results apply to the most general definition of {\em alternative-dependent} distances, where $d(X,Y)$ can also depend on the contents of  $X\setminus Y$, $Y\setminus X$, $X$, and $Y$.

\subsection{Evaluating ABCC rules against noise models} We aim to evaluate the effectiveness of ABCC rules when applied to random profiles generated by large classes of noise models. To this end, we use {\em accuracy in the limit} as a measure.

\begin{definition}[accuracy in the limit]\label{def:acc}
	An ABCC rule $f$ is called {\em accurate in the limit} for a noise model $\mathcal{M}$ if {for every $\varepsilon>0$ there exists $n_\varepsilon$ such that, for every profile with at least $n_\varepsilon$ approval votes produced by $\mathcal{M}$ with ground truth $U$, $f$ returns $U$ as the unique winning committee with probability at least $1-\varepsilon$.}
\end{definition}

Then, ABCC rules are evaluated in terms of {\em robustness} using the next definition.
\begin{definition}[robustness]
Let $d$ be a distance metric over sets of alternatives. An ABCC rule $f$ is monotone robust against $d$ (or $d$-monotone robust) if it is accurate in the limit for all $d$-monotonic noise models.
\end{definition}

We remark that even though we follow the standard definition (see~\citealp{LS18}) according to which an ABCC rule may return more than one winning committees, our definition of the accuracy in the limit (Definition~\ref{def:acc}) is particularly demanding and requires from the ABCC rule to return a {\em unique} committee with high probability. Our purpose here is to guarantee the maximum level of robustness. 

\section{MC is a Uniquely Robust ABCC Rule}\label{sec:mc}
We begin our technical exposition by identifying the unique ABCC rule that is monotone robust against {\em all} distance metrics. Our proofs, in the current and subsequent sections, make extensive use of the following lemma. The notation $S\sim \mathcal{M}(U)$ indicates that the random set $S$ is drawn from the noise model $\mathcal{M}$ with ground truth $U$.
\begin{lemma}\label{lem:suv}
{Let $\mathcal{M}$ be a noise model. An ABCC rule $f$ is
	\begin{enumerate}
		\item [a.] accurate in the limit for $\mathcal{M}$ if $\E_{S\sim \mathcal{M}(U)}[\scf_f(U,S)-\scf_f(V,S)]>0$ for every two different sets of alternatives $U,V\subseteq A$ with $|U|=|V|=k$.
		\item [b.] not accurate in the limit for $\mathcal{M}$ if~~$\E_{S\sim \mathcal{M}(U)}[\scf_f(U,S)-\scf_f(V,S)]<0$ for some pair of sets of alternatives $U,V\subseteq A$ with $|U|=|V|=k$.
		\end{enumerate}		
}
\end{lemma}

\begin{proof}
In the proof, we will use the following variant of the Hoeffding inequality.

\begin{lemma}[\citealp{H63}]\label{lem:hoeffding}
Let $X_1$, $X_2$, ..., $X_\ell$ be i.i.d. random variables with $\E[X_i]=\mu$ and $X_i\in [a,b]$ for $i=1, ..., \ell$, and $X=\sum_{i\in [\ell]}{X_i}$. Then, 
\begin{align*}
    \prm[|X-\ell \mu| \geq t]\leq 2\exp\left(-\frac{2 t^2}{\ell(b-a)^2}\right).
\end{align*}
\end{lemma}

We will need some additional general notation. Denote by $a'$ and $b'$ the minimum and maximum values of the quantity $f(x_1,y)-f(x_2,y)$ over all triplets of integers $x_1$, $x_2$, and $y$ that define pairs $(x_1,y), (x_2,y)\in \mathcal{X}_{m,k}$. Also, for two $k$-sized sets of alternatives $U$ and $V$, define $\mu(U,V)=\E_{S\sim \mathcal{M}(U)}[\scf_f(U,S)-\scf_f(V,S)]$.

\paragraph{Part (a).} Define $\mu_{\min}$ as the minimum among all values $\mu(U,V)$ for the different $k$-sized sets of alternatives $V$ that are different than $U$. By the assumption of part (a) of the lemma, we have $\mu_{\min}>0$. Let $\varepsilon>0$ and 
\begin{align*}
    n_\varepsilon &= \frac{(b'-a')^2}{2\mu_{\min}^2}\ln\frac{2m^k}{\varepsilon}.
\end{align*}
We will first show that for every profile $\Pi=(S_i)_{i\in [n]}$ with at least $n_\varepsilon$ approval votes from the noise model $\mathcal{M}$ with ground truth $U$, the probability that rule $f$ returns $U$ as the unique winner is at least $1-\varepsilon$. 

First observe that
\begin{align*}
    \scf_f(U,\Pi)-\scf_f(V,\Pi) &= \sum_{i\in [n]}{\left(\scf_f(U,S_i)-\scf_f(V,S_i)\right)}
\end{align*}
for every $k$-sized set of alternatives $V$. The quantity $\scf_f(U,S_i)-\scf_f(V,S_i)$ is a random variable following the distribution of $\scf_f(U,S)-\scf_f(V,S)$, where the set $S$ is drawn randomly from the noise model $\mathcal{M}$ with ground truth $U$. Also, observe that the random variable $\scf_f(U,S)-\scf_f(V,S)$ takes values in $[a',b']$. Hence, the score difference $\scf_f(U,\Pi)-\scf_f(V,\Pi)$ is a sum of i.i.d random variables, each with expectation $\mu(U,V)$ and taking values in $[a',b']$.

Hence, we can apply Hoeffding inequality (Lemma~\ref{lem:hoeffding}) with $X=\scf_f(U,\Pi)-\scf_f(V,\Pi)$, $\ell=n$, $a=a'$, $b=b'$, and $\mu=\mu(U,V)$ to get that the probability that $\scf_f(U,\Pi)\leq \scf_f(V,\Pi)$ is
\begin{align*}
    \prm[\scf_f(U,\Pi)\leq \scf_f(V,\Pi)]&= \prm[\scf_f(U,\Pi)-\scf_f(V,\Pi)\leq 0]\\
    &\leq \prm[|\scf_f(U,\Pi)-\scf_f(V,\Pi)-n\mu(U,V)|\geq n\mu(U,V)]\\
    &\leq 2\exp\left(-\frac{2n\mu(U,V)^2}{(b'-a')^2}\right)\leq 2\exp\left(-\frac{2n_\varepsilon\mu_{\min}^2}{(b'-a')^2}\right) \leq  \frac{\varepsilon}{m^k}.
\end{align*}
The second last inequality follows since $n\geq n_\varepsilon$ and $\mu(U,V)\geq \mu_{\min}$ and the last one by the definition of $n_\varepsilon$. 

So far, we have proved that the probability that the score of set $U$ is not higher than the score of set $V$ under $f$ is at most $\varepsilon/m^k$. Hence, the probability that the score of $U$ is not higher than the score of any of the other at most $m^k$ $k$-sized sets of alternatives is at most $\varepsilon$. In other words, $U$ is the unique winner under $f$ with probability at least $1-\varepsilon$ as the definition of the accuracy in the limit (Definition~\ref{def:acc}) requires. This completes the proof of  part (a).

\paragraph{Part (b).} We will again consider a profile $\Pi=(S_i)_{i\in [n]}$ of $n$ approval votes from the noise model $\mathcal{M}$ with ground truth $U$, and show that, as $n$ approaches infinity, the probability that the score of $U$ under the ABCC rule $f$ is strictly lower than that of set $V$, with a probability that approaches $1$. 

Indeed, by applying the Hoeffding inequality for the random variable $X= \scf_f(U,\Pi)-\scf_f(V,\Pi)$, using  $\ell=n$, $a=a'$, $b=b'$, and $\mu=\mu(U,V)$ (notice that $\mu(U,V)<0$ now), 
we get
\begin{align*}
    \prm[\scf_f(U,\Pi)\geq \scf_f(V,\Pi)] &= \prm[\scf_f(U,\Pi)-\scf_f(V,\Pi)\geq 0]\\
    &\leq \prm[|\scf_f(U,\Pi)-\scf_f(V,\Pi)-n\mu(U,V)|\geq -n\mu(U,V)]\\
    &\leq 2\exp\left(-\frac{2n\mu(U,V)^2}{(b'-a')^2}\right),
\end{align*}
which approaches $0$ as $n$ approaches infinity. 
\end{proof}

We are ready to present our first application of Lemma~\ref{lem:suv}.

\begin{theorem}\label{thm:mc}
	MC is the only ABCC rule that is monotone robust against any distance metric.
\end{theorem}

\begin{proof}
	Let $\mathcal{M}$ be a noise model that is $d$-monotonic for some distance metric $d$. Let $U,V\subseteq A$ be any two different $k$-sized sets of alternatives. By the definition of MC, we have
	\begin{align*}
	\E_{S\sim \mathcal{M}(U)}[\scf_{\MC}(U,S)-\scf_{\MC}(V,S)] &= \prm_\mathcal{M}[U|U]-\prm_\mathcal{M}[V|U]>0.
	\end{align*}
	By Lemma~\ref{lem:suv}a, we obtain that MC is $d$-monotone robust.

	We will now show that MC is the only ABCC rule $f$ that has this property. Let $f$ be an ABCC rule that is different than MC. This means that there exist integers $x^*$ and $y^*$ with $(x^*-1,y^*), (x^*,y^*)\in \mathcal{X}_{m,k}$, $(x^*,y^*)\not=(k,k)$, and $f(x^*,y^*)>f(x^*-1,y^*)$. We will construct a distance metric $d$ and a $d$-monotonic noise model for which $f$ is not accurate in the limit.\footnote{We remark that the distance metric $d$ in the second part of the proof of Theorem~\ref{thm:mc} is alternative-dependent. This is necessary; see the discussion in Section~\ref{sec:open}.}
	
	Rename the alternatives of $A$ as $a_1, a_2, ..., a_m$ and let $U=\{a_1, a_2, ..., a_k\}$, $V=\{a_2, ..., a_{k+1}\}$, and $W=\{a_{k-x^*+2}, ..., a_{y^*+k-x^*+1}\}$. Notice that, by the definition of $\mathcal{X}_{m,k}$, $(x^*-1,y^*)\in \mathcal{X}_{m,k}$ implies that $1+\max\{y^*+k-m,0\}\leq x^*$ and, equivalently, $y^*+k-x^*+1\leq m$; hence, the set $W$ is well-defined. Clearly, $x^*\geq 1$; so sets $V$ and $W$ share at least one alternative.
	
	We define a distance metric $d$ between subsets of $A$ that has $d(X,Y)=0$ if $X=Y$, $d(X,Y)\in \{1,2\}$, otherwise, and in particular $d(U,V)=d(U,W)=1$ and $d(U,S)=2$ for every $S$ different than $U$, $V$, or $W$.
	
	We are ready to define the $d$-monotonic noise model $\mathcal{M}$. For simplicity, we use $p_0=\prm_\mathcal{M}[U|U]$, $p_1=\prm_\mathcal{M}[V|U]=\prm_\mathcal{M}[W|U]$, and $p_2=\prm_\mathcal{M}[S|U]$ for every other set $S\subseteq A$ different than $U$, $V$, or $W$. For $0<\delta<\frac{1}{3(2^m-1)}$, we set $p_0=1/3$, $p_1=1/3-\delta$, and $p_2=\frac{2\delta}{2^m-3}$. The particular value of $\delta$ will be specified shortly; for the moment, the range of $\delta$ guarantees that $p_0>p_1>p_2$ so that $\mathcal{M}$ is indeed $d$-monotonic.
	
	We now compute the quantity $\E_{S\sim\mathcal{M}(U)}[\scf_f(U,S)-\scf_f(V,S)]$; observe that $\scf_f(U,U)=\scf_f(V,V)=f(k,k)$, $\scf_f(U,V)=\scf_f(V,U)=f(k-1,k)$, $\scf_f(U,W)=f(x^*-1,y^*)$, and $f(V,W)=f(x^*,y^*)$. We obtain
	\begin{align}\nonumber
	& \E_{S\sim\mathcal{M}(U)}[\scf_f(U,S)-\scf_f(V,S)]\\\nonumber
	&= f(k,k) p_0+f(k-1,k) p_1+f(x^*-1,y^*) p_1 + \sum_{S\not=U,V,W}{f(|U\cap S|,|S|)p_2}-f(k-1,k) p_0\\\nonumber
	&\quad -f(k,k) p_1-f(x^*,y^*) p_1 - \sum_{S\not=U,V,W}{f(|V\cap S|,|S|)p_2}\\\nonumber
	&\leq (p_0-p_1)(f(k,k)-f(k-1,k))-p_1 (f(x^*,y^*)-f(x^*-1,y^*))+p_2\sum_{S\not=U,V,W}{f(|U\cap S|,|S|)}\\\nonumber
	&= \delta (f(k,k)-f(k-1,k))-(1/3-\delta) (f(x^*,y^*)-f(x^*-1,y^*))\\\label{eq:after-delta}
	& \quad +\frac{2\delta}{2^m-3} \sum_{S\not=U,V,W}{f(|U\cap S|,|S|)}.
	\end{align}
	Observe that the RHS of (\ref{eq:after-delta}) is increasing in $\delta$ and approaches $-\frac{1}{3}(f(x^*,y^*)-f(x^*-1,y^*))<0$ as $\delta$ approaches $0$. Hence, for a sufficiently small positive $\delta$, we have
	\begin{align*}
	\E_{S\sim\mathcal{M}(U)}[\scf_f(U,S)-\scf_f(V,S)]<0.
	\end{align*}
	By~Lemma~\ref{lem:suv}b, $f$ is not accurate in the limit for $\mathcal{M}$ and, hence, not monotone robust against the distance metric $d$.
\qed\end{proof}

\section{A Characterization for AV}\label{sec:av}
In this section, we identify the class of distance metrics against which AV is monotone robust. We will need some additional notation that will be useful in several proofs.

For a distance metric $d$ and a set of alternatives $U$, let $\spn(d,U)$ be the number of different non-zero values the quantity $d(U,\cdot)$ can take. We denote these different distance values by $\delta_1(d,U)$, $\delta_2(d,U)$, ..., $\delta_{\spn(d,U)}(d,U)$. We also use $\delta_0(d,U)=0$.
For $t=0, 1, ..., \spn(d,U)$ and alternatives $a,b\in A$, we denote by $N^t_{a|b}(d,U)$ the class of sets of alternatives $S$ that contain alternative $a$ but not alternative $b$ and satisfy $d(U,S)\leq \delta_t(d,U)$.

\begin{definition}[majority-concentricity]
	A distance metric $d$ is called majority-concentric\footnote{Majority-concentricity is similar in spirit with a property of distance metrics over rankings with the same name in~\citep{CPS16}.} if for every $k$-sized set of alternatives $U$, it holds $N^t_{a|b}(d,U)\geq N^t_{b|a}(d,U)$ for every alternatives $a\in U$ and $b\not\in U$ and $t=0, 1, ..., \spn(d,U)$.
\end{definition}

We are ready to prove our characterization for AV.
\begin{theorem}\label{thm:av}
	AV is $d$-monotone robust if and only if the distance metric $d$ is majority-concentric.
\end{theorem}

\begin{proof}
	Let $\mathcal{M}$ be a $d$-monotonic noise model for a majority concentric distance metric $d$. Let $U$ and $V$ be two different sets with $k$ alternatives each. By Lemma~\ref{lem:suv}a, in order to show that AV is accurate in the limit for $\mathcal{M}$ (and, consequently, $d$-monotone robust), it suffices to show that $\E_{S\sim\mathcal{M}(U)}[\scf_{\AV}(U,S)-\scf_{\AV}(V,S)]>0$.
	
	We will need some additional notation. For $t=0, 1, ..., \spn(d,U)$, we denote by $\bar{N}^t(d,U)$) the class of sets of alternatives $S$ that satisfy $d(U,S)=\delta_t(d,U)$. For alternatives $a,b\in A$, we denote $\bar{N}^t_{a}(d,U)$) the subclass of $\bar{N}^t(d,U)$ consisting of sets of alternatives that include $a$ and by $\bar{N}^t_{a|b}(d,U)$ the subclass of $\bar{N}^t_a(d,U)$ consisting of sets do not contain alternative $b$.

	To simplify notation, we set $s=\spn(d,U)$. Also, we drop $(d,U)$ from notation (e.g., we use $N^t_{a|b}$ instead of $N^t_{a|b}(d,U)$) since it is clear from context.	
	We have
\begin{align}\nonumber
    \E_{S\sim \mathcal{M}(U)}[\scf_{\AV}(U,S)]&=\sum_{S\subseteq A}{\scf_{\AV}(U,S)\cdot \prm_\mathcal{M}[S|U]} = \sum_{S\subseteq A}{|U\cap S|\cdot \prm_\mathcal{M}[S|U]}\\\label{eq:triple-sum}
    &=\sum_{a\in U}{\sum_{S\subseteq A:a\in S}{\prm_\mathcal{M}[S|U]}}=\sum_{a\in U}{\sum_{t=0}^{s}{\sum_{S\in \bar{N}^t_a}{\prm_\mathcal{M}[S|U]}}}.
\end{align}
Now, observe that the probability $\prm_\mathcal{M}[S|U]$ is the same for all sets $S\in \bar{N}^t$. In the following, we use $p_t=\prm_\mathcal{M}[S|U]$ for all $S\in \bar{N}^t$, for $t=0, 1, ..., s$. Hence, (\ref{eq:triple-sum}) becomes
\begin{align*}
\E_{S\sim \mathcal{M}(U)}[\scf_{\AV}(U,S)]
	&= \sum_{a\in U}{\sum_{t=0}^{s}{|\bar{N}^t_a|\cdot p_t}}
\end{align*}
Similarly, we have
\begin{align*}
\E_{S\sim \mathcal{M}(U)}[\scf_{\AV}(V,S)] &= \sum_{a\in V}{\sum_{t=0}^{s}{|\bar{N}^t_a|\cdot p_t}},
\end{align*}
and, by linearity of expectation,
\begin{align}\label{eq:difference}
 \E_{S\sim \mathcal{M}(U)}[\scf_{\AV}(U,S)-\scf_{\AV}(V,S)] &= \sum_{a\in U\setminus V}{\sum_{t=0}^{s}{|\bar{N}^t_a|\cdot p_t}}-\sum_{a\in V\setminus U}{\sum_{t=0}^{s}{|\bar{N}^t_a|\cdot p_t}}.
\end{align}
Let $\mu$ be a bijection that maps each alternative of $U\setminus V$ to a distinct alternative of $V\setminus U$. Then, (\ref{eq:difference}) becomes
\begin{align}\nonumber
    & \E_{S\sim \mathcal{M}(U)}[\scf_{\AV}(U,S)-\scf_{\AV}(V,S)]\\\nonumber
    &= \sum_{a\in U\setminus V}{\sum_{t=0}^{s}{|\bar{N}^t_a|\cdot p_t}}-\sum_{a\in U\setminus V}{\sum_{t=0}^{s}{|\bar{N}^t_{\mu(a)}|\cdot p_t}} =\sum_{a\in U\setminus V}{\sum_{t=0}^{s}{\left(|\bar{N}^t_{a|\mu(a)}|-|\bar{N}^t_{\mu(a)|a}|\right)\cdot p_t}}\\\nonumber
    &=\sum_{a\in U\setminus V}{\left(|N^0_{a|\mu(a)}|-|N^0_{\mu(a)|a}|\right)\cdot p_0} + \sum_{a\in U\setminus V}{\sum_{t=1}^{s}{\left(|N^t_{a|\mu(a)}|-|N^{t-1}_{a|\mu(a)}|-|N^t_{\mu(a)|a}|+|N^{t-1}_{\mu(a)|a}|\right)\cdot p_t}}\\\label{eq:differences}
    &= \sum_{a\in U\setminus V}{\sum_{t=0}^{s-1}{\left(|N^t_{a|\mu(a)}|-|N^t_{\mu(a)|a}|\right)\cdot (p_t-p_{t+1})}}+\left(|N^s_{a|\mu(a)}|-|N^s_{\mu(a)|a}|\right)\cdot p_s\\\nonumber
	&\geq \sum_{a\in U\setminus V}{\left(|N^0_{a|\mu(a)}|-|N^0_{\mu(a)|a}|\right)\cdot (p_0-p_1)} >0.
\end{align}
The third equality follows since $\bar{N}^0_{a|\mu(a)}=N^0_{a|\mu(a)}$, $\bar{N}^0_{\mu(a)|a}=N^0_{\mu(a)|a}$, and $\bar{N}^t_{a|\mu(a)}=N^t_{a|\mu(a)}\setminus N^{t-1}_{a|\mu(a)}$ and $\bar{N}^t_{\mu(a)|a}=N^t_{\mu(a)|a}\setminus N^{t-1}_{\mu(a)|a}$ for $t=1, ..., s$. The first inequality follows since $d$ is majority concentric and since $p_t>p_{t+1}$ and, thus, all differences in (\ref{eq:differences}) are non-negative. The last inequality follows after observing that  $|N^0_{a|\mu(a)}|=1$ and $|N^0_{\mu(a)|a}|=0$ for $a\in U\setminus V$ and since $p_0>p_1$. This completes the ``if'' part of the proof.

Let us now consider a non-majority concentric distance metric $d$ that satisfies $N^{t^*}_{a|b}(d,U)< N^{t^*}_{b|a}(d,U)$ for the $k$-sized set of alternatives $U$, some alternatives $a\in U$ and $b\not\in U$, and some $t^*\in \{1, 2, ..., \spn(d,U)-1\}$. We show the ``only if'' part of the theorem by constructing a noise model $\mathcal{M}$ that satisfies $\E_{S\sim \mathcal{M}(U)}[\scf_{\AV}(U,S)-\scf_{\AV}(V,S)]{<} 0$ for $V=U\setminus\{a\}\cup \{b\}$.

Again, we use $p_t=\prm_\mathcal{M}[S|U]$ for every set of alternatives $S\in \bar{N}^t(d,U)$, $s=\spn(d,U)$, and drop $(d,U)$ from notation. We define the model probabilities so that $\tau=p_0>p_1>...>p_{t^*}=\tau-\epsilon$ and $2\epsilon=p_{t^*+1}>...>p_{s}=\epsilon$. Notice that such a noise model exists for any arbitrarily small $\epsilon>0$. Since there are $2^m$ sets of alternatives and $\tau$ is the probability that $\mathcal{M}$ returns the ground truth ranking, it must be $\tau>1/2^m$. We now apply equality (\ref{eq:differences}). Observe that, since $V=U\setminus \{a\}\cup \{b\}$, we have $\mu(a)=b$. We obtain
\begin{align*}\nonumber
 \E_{S\sim \mathcal{M}(U)}[\scf_{\AV}(U,S)-\scf_{\AV}(V,S)] &= \sum_{t=0}^{s-1}{\left(|N^t_{a|b}|-|N^t_{b|a}|\right)\cdot (p_t-p_{t+1})}
+\left(|N^s_{a|b}|-|N^s_{b|a}|\right)\cdot p_s\\\nonumber
&=\sum_{t=0}^{t^*-1}{\left(|N^t_{a|b}|-|N^t_{b|a}|\right)\cdot (p_t-p_{t+1})}+\left(|N^{t^*}_{a|b}|-|N^{t^*}_{b|a}|\right)\cdot (p_{t^*}-p_{t^*+1})\\
&\quad +\sum_{t=t^*+1}^{s-1}{\left(|N^t_{a|b}|-|N^t_{b|a}|\right)\cdot (p_t-p_{t+1})}+\left(|N^s_{a|b}|-|N^s_{b|a}|\right)\cdot p_s.
\end{align*}
Now, observe that for $t\not=t^*$, it holds $|N^t_{a|b}|-|N^t_{b|a}|\leq 2^m$ (the total number of sets of alternatives) and $p_t-p_{t+1}\leq \epsilon$. Also, $|N^{t^*}_{a|b}|-|N^{t^*}_{b|a}|\leq -1$ and $p_{t^*}-p_{t^*+1}=\tau-3\epsilon$. Setting specifically $\epsilon= \frac{1}{s 8^m}$, we obtain that
\begin{align*}
\E_{S\sim \mathcal{M}(U)}[\scf_{\AV}(U,S)-\scf_{\AV}(V,S)] &\leq s2^m\epsilon-(\tau-3\epsilon)\leq \frac{1}{4^m}-\frac{1}{2^m}+\frac{3}{s \cdot 8^m},
\end{align*}
which is negative for $m\geq 2$ since $s\geq 1$. The proof of the ``only if'' part of Theorem~\ref{thm:av} now follows by Lemma~\ref{lem:suv}b.
\end{proof}

It is tempting to conjecture that AV and MC are the only ABCC rules that are monotone robust against all majority concentric distance metrics. However, this is not true as the next example, which uses a different ABCC rule, shows.

\begin{example}
Let $A=\{a,b,c\}$ and $k=2$. Consider the majority concentric distance metric $d$ and the ABCC rule $f$ with $f(1,1)=1$, $f(2,2)=2$, and $f(x,y)=0$ otherwise. Despite its similarity with AV, the rule $f$ is different since $f(1,2)=0$. We will show that $f$ is $d$-monotone robust against any majority concentric distance metric $d$. Without loss of generality, let us assume that $U=\{a,b\}$ and $V=\{a,c\}$. Observe that the quantity $\scf_f(U,S)-\scf_f(V,S)$ is equal to $0$ when $S=\emptyset, \{a\}, \{b,c\},\{a,b,c\}$, $1$ when $S=\{b\}$, $-1$ when $S=\{c\}$, $2$ when $S=\{a,b\}$, and $-2$ when $S=\{a,c\}$. Hence, for the $d$-monotonic noise model $\mathcal{M}$, we have  $\E_{S\sim\mathcal{M}(U)}[\scf_f(U,S)-\scf_f(V,S)]=2p_{ab}-2p_{ac}+p_b-p_c$, where $p_{ab}$, $p_{ac}$, $p_b$, and $p_c$ are abbreviations for the probabilities $\prm_\mathcal{M}[S|U]$ for $S=\{a,b\}$, $\{a,c\}$, $\{b\}$, and $\{c\}$, respectively.

In order to have $N^t_{b|c}\geq N^t_{c|b}$ for $t=0, 1, ..., \spn(d,U)$ as the definition of majority concentricity requires, it must be either $p_{ab}>p_b,p_c\geq p_{ac}$ or $p_{ab}>p_b,p_{ac}\geq p_c$. In the first case, we have $\E_{S\sim\mathcal{M}(U)}[\scf_f(U,S)-\scf_f(V,S)]=(p_{ab}-p_{ac})+(p_{ab}-p_c)+(p_b-p_{ac})>0$. In the second case, we have $\E_{S\sim\mathcal{M}(U)}[\scf_f(U,S)-\scf_f(V,S)]=2(p_{ab}-p_{ac})+(p_{b}-p_c)>0$. Accuracy in the limit of the ABCC rule $f$ for the noise model $\mathcal{M}$ then follows by Lemma~\ref{lem:suv}a.
\qed\end{example}

\section{Robustness of Other ABCC Rules}\label{sec:abc}
The definitions, statements, and proofs that we present in this section use appropriately defined bijections on sets of alternatives.

\begin{definition}
Given two different sets $U$ and $V$ with $|U|=|V|$, a $(U,V)$-bijection $\mu:2^A\rightarrow 2^A$ is defined as $\mu(S)=\{\mu'(a): a\in S\}$, where $\mu':A\rightarrow A$ is such that $\mu'(a)=a$ for every alternative $a\in U \cap V$ or $a\not\in U\cup V$, $\mu'(a)$ is a distinct alternative in $V\setminus U$ for $a\in U\setminus V$, and $\mu'(a)$ is a distinct alternative in $U\setminus V$ for $a\in V\setminus U$.
\end{definition}

It is easy to see that a $(U,V)$-bijection $\mu$ has the following properties.

\begin{lemma}\label{lem:bijection}
	Let $U,V\subseteq A$ with $|U|=|V|$ and let $\mu$ be a $(U,V)$-bijection. For every $S\subseteq A$, it holds $|S|=|\mu(S)|$, $|U\cap S|=|V\cap\mu(S)|$, and $|U\cap \mu(S)|=|V\cap S|$.
\end{lemma}

\begin{proof}
The proof follows easily by the definition of the $(U,V)$-bijection $\mu$. The equality $|S|=|\mu(S)|$ holds because the function $\mu'$ maps each alternative of $S$ to a distinct alternative. To prove the second equality, observe that the function $\mu'$ maps each alternative in $U\cap V$ to itself, each alternative in $U\setminus V$ to a distinct alternative in $V\setminus U$, and each alternative not belonging to $U$ to an alternative not belonging to $V$. Hence, the alternatives in $U\cap S$ (and no other alternative in $S$) are mapped to distinct alternatives of $V$ and $|U\cap S|=|V\cap \mu(S)|$. The proof of the third equality is symmetric. 
\qed\end{proof}

We are now ready to proceed with the presentation of our last set of results. In particular, our results for ABCC rules different than MC and AV involve two classes of distance metrics. We define the first one here.

\begin{definition}[natural distance]
	A distance metric $d$ is called natural if for every three sets $U$, $V$, and $S$ with $|U|=|V|$ such that $|U\cap S|>|V\cap S|$, it holds that $d(U,S)\leq d(V,S)$.	
\end{definition}

The next observation follows easily by the definitions.

\begin{lemma}\label{lem:natural-is-majority-concentric}
	Any natural distance metric is majority-concentric.
\end{lemma}

\begin{proof}
Let $d$ be a natural distance, $U$ a $k$-sized set of alternatives, and $a,b\in A$ with $a\in U$ and $b\not\in U$. We will show that $|N^t_{a|b}(d,U)|\geq |N^t_{b|a}(d,U)|$ for $t=0,1, ..., \spn(d,U)$. For $t=0$, this is clearly true since $N^0_{a|b}(d,U)=\{U\}$ and $N^0_{b|a}(d,U)=\emptyset$.
	
For $t\geq 1$, let $V=U\setminus\{a\}\cup \{b\}$ and $\mu$ be any $(U,V)$-bijection on sets of alternatives. Let $S\in N^t_{b|a}(d,U)$. By the definition of $\mu$, $\mu(S)$ contains alternative $a$ but not $b$. Also $|U\cap \mu(S)|=|U\cap S|+1$ and, due to naturality of $d$, $d(U,\mu(S))\leq d(U,S)$. We conclude that $\mu(S)\in N^t_{a|b}(d,U)$. Since $\mu$ is a bijection (the sets of $N^t_{b|a}(d,U)$ are mapped to distinct sets in $N^t_{a|b}(d,U)$), we get $|N^t_{a|b}(d,U)|\geq |N^t_{b|a}(d,U)|$, as desired.
\end{proof}

The opposite is not true as the next example illustrates.

\begin{example}
	Let $A=\{a,b,c\}$ and consider the distance metric with $d(X,Y)=0$ for every pair of sets with $X=Y$, $d(X,Y)=1$ if $X\cap Y=\emptyset$ and $X\cup Y=A$, and $d(X,Y)=2$, otherwise. It can be easily seen that the distance is majority-concentric; it suffices to observe that, within distance $1$ from any set, each alternative appears in exactly one set. To see that is not natural, consider $U=\{a,b\}$, $V=\{a,c\}$ and $S=\{b\}$. We have $|U\cap S|>|V \cap S|$ but $d(U,S)=2>1=d(V,S)$.
\qed\end{example}

Our next lemma identifies the class of ABCC rules that are monotone robust against all natural distance metrics. 

\begin{lemma}\label{lem:non-trivial-natural}
{An ABCC rule is $d$-monotone robust against a natural distance metric $d$ if for every two different sets of alternatives $U,V\subseteq A$ with $|U|=|V|=k$ there exists a $(U,V)$-bijection $\mu$ on sets of alternatives and a set $S\subseteq A$ with $\scf_f(U,S)>\scf_f(V,S)$ and $d(U,S)<d(U,\mu(S))$.}
\end{lemma}

\begin{proof}
	Let $U$ and $V$ be two different sets with $k$ alternatives each. Let $\mathcal{S}_{+}$, $\mathcal{S}_{-}$, and $\mathcal{S}_0$ be the classes of sets of alternatives $S$ with $|U\cap S|>|V\cap S|$, $|U\cap S|<|V\cap S|$, and $|U\cap S|=|V\cap S|$, respectively. Using this notation, we have
\begin{align}\nonumber
\E_{S\sim \mathcal{M}(U)}[\scf_f(U,S)-\scf_f(V,S)] &= \sum_{S\subseteq A}{\left(\scf_f(U,S)-\scf_f(V,S)\right)\cdot \prm_\mathcal{M}[S|U]}\\\nonumber
&= \sum_{S\in \mathcal{S}_{+}}{\left(\scf_f(U,S)-\scf_f(V,S)\right)\cdot \prm_\mathcal{M}[S|U]}\\\nonumber
&\quad+\sum_{S\in \mathcal{S}_{0}}{\left(\scf_f(U,S)-\scf_f(V,S)\right)\cdot \prm_\mathcal{M}[S|U]}\\\label{eq:three-sums}
&\quad+\sum_{S\in \mathcal{S}_{-}}{\left(\scf_f(U,S)-\scf_f(V,S)\right)\cdot \prm_\mathcal{M}[S|U]}
\end{align}
We will now transform the third sum in the RHS of (\ref{eq:three-sums}) to one running over the sets of $\mathcal{S}_{+}$ like the first sum.

Let $\mu$ be a $(U,V)$-bijection on sets of alternatives; by Lemma~\ref{lem:bijection}, $\mu$ maps every set of $\mathcal{S}_-$ to a set of $\mathcal{S}_{+}$ and vice-versa. Hence, instead of enumerating sets of $\mathcal{S}_-$, we could enumerate sets of $\mathcal{S}_+$ and apply the bijection $\mu$ on them. The third sum in the RHS of (\ref{eq:three-sums}) then becomes
\begin{align}\nonumber
\sum_{S\in \mathcal{S}_{-}}{\left(\scf_f(U,S)-\scf_f(V,S)\right)\cdot \prm_\mathcal{M}[S|U]} &= \sum_{S\in \mathcal{S}_{+}}{\left(\scf_f(U,\mu(S))-\scf_f(V,\mu(S))\right)\cdot \prm_\mathcal{M}[\mu(S)|U]}\\\label{eq:third-sum}
&= \sum_{S\in \mathcal{S}_{+}}{\left(\scf_f(V,S)-\scf_f(U,S)\right)\cdot \prm_\mathcal{M}[\mu(S)|U]}
\end{align}
The second equality follows since, by Lemma~\ref{lem:bijection}, $\scf_f(U,\mu(S))=f(|U\cap \mu(S)|,|\mu(S)|)=\scf_f(|V\cap S|,|S|)=\scf_f(V,S)$ and, similarly, $\scf_f(V,\mu(S))=\scf_f(U,S)$.

Now observe that $\scf_f(U,S)=\scf_f(V,S)$ when $S\in \mathcal{S}_0$. Hence, the second sum in the RHS of (\ref{eq:three-sums}) is equal to $0$. By combining (\ref{eq:three-sums}) and (\ref{eq:third-sum}), we get
\begin{align}\label{eq:friendly-f-d}
\E_{S\sim \mathcal{M}(U)}[\scf_f(U,S)-\scf_f(V,S)] &=\sum_{S\in \mathcal{S}_{+}}{\left(\scf_f(U,S)-\scf_f(V,S)\right)\cdot (\prm_\mathcal{M}[S|U]-\prm_\mathcal{M}[\mu(S)|U])}.
\end{align}

Now observe that the RHS of (\ref{eq:friendly-f-d}) is always non-negative. This is due to the fact that $S\in \mathcal{S}_+$ which implies that $\scf_f(U,S)=f(|U\cap S|,|S|)\geq f(|V\cap S|,|S|)=\scf_f(V,S)$ since $f$ is non-decreasing in its first argument and $d(U,S)\leq d(U,\mu(S))$ (and, consequently, $\prm_\mathcal{M}[S|U]\geq \prm_\mathcal{M}[\mu(S)|U]$) since $d$ is natural {and, by Lemma~\ref{lem:bijection}, $|U\cap S|>|V\cap S|=|U\cap \mu(S)|$.} {The RHS of (\ref{eq:friendly-f-d}) is strictly positive if there exists a set $S\in \mathcal{S}_+$  that further satisfies $d(U,S)<d(U,\mu(S))$ (and, consequently, $\prm_\mathcal{M}[S|U]> \prm_\mathcal{M}[\mu(S)|U]$) and $\scf_f(U,S)>\scf_f(V,S)$}. The lemma then follows by Lemma~\ref{lem:suv}a.
\end{proof}

We now present two applications of Lemma~\ref{lem:non-trivial-natural}. 

\begin{theorem}\label{thm:jump}
	An ABCC rule $f$ is monotone robust against any natural distance metric if and only if $f(k,k)>f(k-1,k)$.
\end{theorem}

\begin{proof}
For the proof of the ``if'' part, we use Lemma~\ref{lem:non-trivial-natural}. Consider any pair of different $k$-sized sets alternatives $U$ and $V$ and any $(U,V)$-bijection $\mu$. For $S=U$, it holds $d(U,U)<d(U,\mu(U))$ and $\scf_f(U,U)>\scf_f(V,U)$ by the definition of $f$.

For the proof of the ``only if'' part, assume that $f(k,k)=f(k-1,k)$ and consider the natural distance metric $d$ with $d(X,Y)=0$ if $X=Y$ and $d(X,Y)=1$ otherwise. Let $U,V\subseteq A$ be different sets with $k$ alternatives each such that $|U\cap V|=k-1$. Let $U\setminus V=\{a\}$ and $V\setminus U=\{b\}$.

Unfortunately, as we will see, we are in the case $\E_{S\sim \mathcal{M}(U)}[\scf_f(U,S)-\scf_f(V,S)]=0$ and, hence, we cannot use Lemma~\ref{lem:suv} to complete the proof. We will instead prove that $\scf_f(U,S)-\scf_f(V,S)$ follows a symmetric distribution (i.e., $\prm_{\mathcal{M}}[\scf_f(U,S)-\scf_f(V,S)=t]=\prm_{\mathcal{M}}[\scf_f(U,S)-\scf_f(V,S)=-t]$ for every $t>0$). Then, the random variable $\scf_f(U,\Pi)-\scf_f(V,\Pi)$, where $\Pi$ is a random profile of approval votes drawn from the noisy model $\mathcal{M}$ with ground truth $U$ will be symmetric with expectation zero as well. Hence, the probability that $\scf_f(U,\Pi)-\scf_f(V,\Pi)$ is strictly positive (which is a necessary condition so that $U$ is the unique winner) is at most $1/2$, and, hence, $f$ is not accurate in the limit for $\mathcal{M}$. 

First observe that $\scf_f(U,S)-\scf_f(V,S)=0$ when $S$ is equal to $U$ or $V$, or contains both alternatives $a$ and $b$, or contains neither $a$ nor $b$. Indeed, we have $\scf_f(U,U)-\scf_f(V,U)=f(k,k)-f(k-1,k)=0$ and $\scf_f(U,V)-\scf_f(V,V)=f(k-1,k)-f(k,k)=0$, by our assumption. Furthermore, if $S$ contains both $a$ and $b$ or none of them, we have that $|U\cap S|=|V\cap S|$ and, hence, $\scf_f(U,S)=\scf_f(V,S)$.

Denote by $\mathcal{S}$ the collection of the remaining sets of alternatives, i.e.,
\begin{align*}
    \mathcal{S} &= \left\{S: S\not=U,V \mbox{ and } |S\cap \{a,b\}|=1\right\},
\end{align*}
and partition $\mathcal{S}$ into the subcollections $\mathcal{S}_{a|b}$ and $\mathcal{S}_{b|a}$ consisting of sets that include alternative $a$ (but not alternative $b$) and alternative $b$ (but not alternative $a$, respectively). Now, consider the (unique) $(U,V)$-bijection $\mu$ and observe that for every set $S$ in $\mathcal{S}_{a|b}$, $\mu(S)$ is a distinct set of $\mathcal{S}_{b|a}$ and vice-versa. Furthermore, notice that, by Lemma~\ref{lem:bijection}, we have \begin{align*}
\scf_f(U,S)-\scf_f(V,S)&=f(|U\cap S|,|S|)-f(|V\cap S|,|S|)\\
&=f(|V\cap \mu(S)|,|\mu(S)|)-f(|U\cap \mu(S)|,|\mu(S)|)\\
&=-(\scf_f(U,\mu(S))-\scf_f(V,\mu(S))).
\end{align*}
The proof completes by observing that, by the definition of the distance metric $d$, the noise model $\mathcal{M}$ with ground truth $U$ returns each set of $\mathcal{S}$ (and, consequently, $S$ and $\mu(S)$) equiprobably.
\end{proof}

Notice that most popular ABCC rules from Section~\ref{sec:prelim} satisfy the condition of Theorem~\ref{thm:jump}. {CC is an exception. The proof of Theorem~\ref{thm:jump} implies that CC is not monotone robust for the natural distance metric $d$ defined as $d(X,Y)=0$ if $X=Y$ and $d(X,Y)=1$, otherwise.}

Our second application of Lemma~\ref{lem:non-trivial-natural} involves all non-trivial ABCC rules and an important subclass of natural distances.

\begin{definition}[similarity distance]
	A natural distance metric $d$ is a similarity distance metric if for every three sets $U$, $V$, and $S$ with $|U|=|V|$ such that $|U\cap S|>|V\cap S|$, it holds that $d(U,S)<d(V,S)$.	
\end{definition}

\begin{theorem}\label{thm:all-non-trivial}
	Any non-trivial ABCC rule is monotone robust against any similarity distance metric.
\end{theorem}

\begin{proof}
	We apply Lemma~\ref{lem:non-trivial-natural} assuming a non-trivial ABCC rule $f$ and a similarity distance metric $d$. Non-triviality of $f$ implies that for every two different sets $U$ and $V$ with $k$ alternatives each, there is a set $S$ such that $\scf_f(U,S)>\scf_f(V,S)$. This yields {$|U\cap S|>|V\cap S|=|U\cap \mu(S)|$, where $\mu$ is any $(U,V)$-bijection (see Lemma~\ref{lem:bijection}), and implies that $d(U,S)<d(U,\mu(S))$ since $d$ is a similarity distance.}
\end{proof}

We can easily show that the four distance metrics set difference, Jaccard, Zelinka, and Bunke-Shearer that we defined in Section~\ref{sec:prelim} are all similarity distance metrics. Using this observation and Theorem~\ref{thm:all-non-trivial}, we obtain the next statement.

\begin{corollary}
Any non-trivial ABCC rule is monotone robust against the set difference, Jaccard, Zelinka, and Bunke-Shearer distance metrics.
\end{corollary}

\section{Discussion and open problems}\label{sec:open}
We believe that our approach complements nicely the axiomatic and quantitative analysis of approval-based multiwinner voting. The main problem that we leave open is whether there are ABCC rules that lie between rules MC and AV in terms of robustness.

This question is related to a subtle issue that involves alternative-independent distance metrics. In Section~\ref{sec:mc}, we showed that the ABCC rule MC is the only one that is monotone robust against any distance metric. In the second part of the proof of Theorem~\ref{thm:mc}, we used an alternative-dependent distance metric $d$ and argued that there is a $d$-monotonic noise model for which mechanism MC is not accurate in the limit. 

Is alternative-dependence really necessary? In the following we show that this is indeed the case by presenting an ABCC rule that is different than MC and monotone robust against any alternative-independent distance metric. Indeed, consider the case with $m=4$, $A=\{a,b,c,d\}$, and $k=2$ and the ABCC rule $f$ defined as follows: for $(x,y)\in \mathcal{X}_{4,2}$, it is $f(x,y)=x$ if $y=2$ and $f(x,y)=0$ otherwise.

\begin{theorem}
The ABCC rule $f$ is monotone robust against any alternative-independent distance metric.
\end{theorem}

\begin{proof}
We consider an alternative-independent distance metric $d$ and the ground truth $U=\{a,b\}$. For a set of alternatives $S\subseteq A=\{a,b,c,d\}$, the distance $d(U,S)$ can be thought of as a function that depends only on $|U\cap S|$ and $|S|$. The $d$-monotonic noise model $\mathcal{M}$ produces the approval set $S$ with probability $\prm_{\mathcal{M}}[S|U]=p(|U\cap S|,|S|)$ when the ground truth $U$ is used. The different probability values that the noise model uses correspond to those pairs $(x,y)$ that belong to the set $\mathcal{X}_{4,2}=\{(0,0),(0,1),(1,1),(0,2),(1,2),(2,2),(1,3),(2,3),(2,4)\}$. 

We will show that $f$ is accurate in the limit for $\mathcal{M}$. By Lemma~\ref{lem:suv} and due to the symmetry implied by alternative-independence, we need to show that $\E_{S\sim \mathcal{M}(U)}[\scf_f(U,S)-\scf_f(V,S)]>0$ for $V=\{a,c\}$ and $V=\{c,d\}$. For $V=\{a,c\}$, we have
\begin{align*}
    \E_{S\sim \mathcal{M}(U)}[\scf_f(U,S)-\scf_f(V,S)] &=
    \sum_{S\subset A}{f(|\{a,b\} \cap S|,|S|)-f(|\{a,c\} \cap S|,|S|)\cdot p(|\{a,b\}\cap S|,|S|)}\\
    &= p(2,2)-p(1,2)>0.
\end{align*}
In the second equality, we have used the fact that the contribution of all sets $S$ of size different than $2$ as well as of the sets $\{a,d\}$, $\{b,c\}$ to the sum is zero. In addition, the contribution of the sets $\{a,b\}$, $\{a,c\}$, $\{b,d\}$, and $\{c,d\}$ is $p(2,2)$, $-p(1,2)$, $p(1,2)$, and $-p(1,2)$. The inequality follows since $p(2,2)$ is the probability with which $\mathcal{M}$ returns the ground.

For $V=\{c,d\}$, we have
\begin{align*}
    \E_{S\sim \mathcal{M}(U)}[\scf_f(U,S)-\scf_f(V,S)] &=
    \sum_{S\subset A}{f(|\{a,b\} \cap S|,|S|)-f(|\{c,d\} \cap S|,|S|)\cdot p(|\{a,b\}\cap S|,|S|)}\\
    &= 2p(2,2)-2p(0,2)>0.
\end{align*}
The contribution to the sum of all sets besides $\{a,b\}$ and $\{c,d\}$ is zero. The contribution of the sets $\{a,b\}$ and $\{c,d\}$ is $2p(2,2)$ and $-2p(0,2)$ respectively. Again, the inequality follows since $p(2,2)$ is the probability with which $\mathcal{M}$ returns the ground.
\end{proof}

Clearly, if we restrict our attention to alternative-independent distance metrics, the rule $f$ is, in addition to MC, another ABCC rule that is strictly more robust than AV. We can furthermore show that there are rules with intermediate robustness. For example, in the same setting with $m=4$ and $k=2$, one such rule $f'$ is defined as $f'(x,y)=f_{\AV}(x,y)$ if $y\not=2$ and $f'(x,2)=2f_{\AV}(x,2)$. It can be shown (the proof is omitted) to be strictly more robust than AV and strictly less robust than MC (and $f$). It would be interesting to obtain general results (for general values of the parameters $m$ and $k$) and characterize all ABCC rules that lie between MC and AV in terms of robustness in alternative-independent distance metrics only.

Furthermore, applying our framework to non-ABCC rules deserves investigation. Beyond assessing the effects of noise in the limit, studying the sample complexity of approval-based multiwinner voting is important. This will require the design of concrete noise models like the $\mathcal{M}_p$ model that we presented in Section~\ref{sec:prelim}. In particular, models that simulate user behaviour in crowdsourcing platforms will be useful for evaluating approval-based voting in such environments. Even though the $\mathcal{M}_p$ model is very simple, we expect that implementation issues will emerge for more elaborate noise models. Similar issues in the implementation of the Mallows ranking model~\citep{M57} have triggered much non-trivial work; see, e.g., \citet{DPR04}.

\section*{Acknowledgments}
This research is co-financed by Greece and the European Union (European Social Fund) through the Operational Programme ``Human Resources Development, Education and Lifelong Learning 2014-2020'' (project MIS 5047146).

\bibliographystyle{plainnat}
\bibliography{abc-full-names}

\appendix

\section{Appendix}

\begin{theorem}
The ABCC rule AV is a maximum likelihood estimator for the noise model $\mathcal{M}_p$.
\end{theorem}

\begin{proof}
Let $\Pi=(S_i)_{i\in [n]}$ be a profile with $n$ approval votes. We need to show that the profile $\Pi$ has maximum probability to have been produced by the noise model $\mathcal{M}_p$ with a committee of maximum AV score from the votes of $\Pi$ as the ground truth. 

Indeed, the probability that $\Pi$ has been produced by the noise model $\mathcal{M}_p$ with ground truth committee $U$ is 
\begin{align*}
    \prod_{i\in [n]}\prm_{\mathcal{M}_p}[S_i|U] &= \prod_{i\in [n]}{p^m\cdot \left(\frac{1-p}{p}\right)^{d_\Delta(S_i,U)}}=p^{mn}\cdot\left(\frac{1-p}{p}\right)^{\sum_{i\in [n]}{d_\Delta(S_i,U)}}.
\end{align*}
Since $p>1/2$, the above expression is maximized by minimizing the quantity $\sum_{i\in [n]}{d_\Delta(S_i,U)}$. Now, observe that
\begin{align*}
    \sum_{i\in [n]}{d_\Delta(U,S_i)} &= \sum_{i\in [n]}{\left(|U\setminus S_i|+|S_i\setminus U|\right)}= \sum_{i\in [n]}\left(|U|+|S_i|-|U\cap S_i|\right)\\
    &= nk+\sum_{i\in [n]}{|S_i|}-\sum_{i\in [n]}{\scf_{AV}(U,S_i)}= nk+\sum_{i\in [n]}{|S_i|}-\scf_{AV}(U,\Pi).
\end{align*}
Hence, the probability that the profile $\Pi$ is generated by a noise model $\mathcal{M}_p$ is maximized for the ground truth committee $U$ of maximums score $\scf_{\AV}(U,\Pi)$.
\end{proof}

\end{document}